\documentclass[twocolumn]{article}
\usepackage[totalwidth=490pt,totalheight=680pt]{geometry}

\usepackage{cite}
\usepackage{graphicx}
\DeclareGraphicsExtensions{.eps}
\usepackage{amsmath}
\usepackage{amssymb}
\usepackage{amsthm}
\usepackage{bbm}

\newtheorem{prop}{Proposition}
\usepackage[tight,footnotesize]{subfigure}
\usepackage{authblk}

\begin{document}

\title{Learning From Missing Data Using Selection Bias in Movie Recommendation}

\author[1]{Claire Vernade}
\author[2]{Olivier Capp\'e}
\affil[1,2]{LTCI, CNRS, T\'el\'ecom ParisTech, Universit\'e Paris-Saclay}
\date{}
\renewcommand\Authands{ and }

\maketitle
	
	\begin{abstract}
	Recommending items to users is a challenging task due to the large
	amount of missing information. In many cases, the data solely consist of
	ratings or tags voluntarily contributed by each user on a very limited subset
	of the available items, so that most of the data of potential interest is
	actually missing. Current approaches to
	recommendation usually assume that the unobserved data is missing at random.
	
	In this contribution, we provide statistical evidence that existing movie
	recommendation datasets reveal a significant positive association between the
	rating of items and the propensity to select these items. We propose a
	computationally efficient variational approach that makes it possible to
	exploit this selection bias so as to improve the estimation of ratings from
	small populations of users. Results obtained with this approach applied to
	neighborhood-based collaborative filtering
	illustrate its potential for improving the reliability of the recommendation.
\end{abstract}

\section{Introduction} 

Since the early 1990's, automated methods for recommending content to users
based on historical data has been an active line of research in connection with
the widespread deployment of online services
\cite{goldberg1992using,resnick1997recommender}. The Netflix prize
\cite{bennett2007netflix,bell2007lessons} was a recent highlight that triggered off a lot of
attention on movie recommendation. In this contribution, we consider settings
that are typical of the collaborative filtering paradigm in which the available
information can be summarized by the list of ratings of some ``items'',
contributed voluntarily by the users. The aim is to exploit these ratings
originating from the whole population so as to recommend items to a specific
user, based on his/her own historical data \cite{herlocker1999algorithmic,herlocker2002empirical}. In practice, these methods are
rarely used alone and can be complemented by using item and/or user metadata or
features so as to improve the prediction.

A frequent abstraction used in this field consists in viewing the data
available at some point in the process as a very large matrix of
ratings, 
whose rows correspond to users and columns to items, that is incompletely
observed. In typical datasets available in movie recommendation, the number of
missing entries from this matrix is two orders of magnitude larger than the
number of entries that are actually observed. The main challenge is thus to
extrapolate the observed ratings despite the very large fraction of missing
data. To do so, it is a standard practice to consider only the ratings that
have been actually observed; the ratings that have not been observed being
simply ignored. Doing so means that the sampling distribution, under which the
ratings in the matrix are revealed by the data is viewed as a nuisance
parameter rather than as an aspect of the data that could be use for estimation
---see, e.g., \cite{tienmai2015bayesian} in
the context of matrix completion.

Our goal with this paper is to investigate the gain achievable by exploiting
the ``selection bias'' that is present in available movie recommendation
datasets. This bias consists in a significant positive association between the
rating of items in a given population and the natural propensity of this  population
to select these items. We will show in Section~\ref{sec:obs-data} below that
this observation is robust, being present both at the scale of whole datasets
but also in ratings corresponding to small sub-populations of users. To leverage this observation in a manner that stays computationally
feasible in realistic scenarios, we will use a simple convex variational
criterion that captures the main features of the relationship between the
ratings and the item popularity. To illustrate the approach, we will specialize
it to the case of neighborhood-based collaborative filtering in
which the preferences of the user is extrapolated from the population of users
that are closest to him/her given his/her historical ratings.

Taking into account the popularity of items in order to improve the
recommendation has been considered before by \cite{bell2007scalable} who design
a greedy sequential preprocessing procedure aimed at subtracting different
explanatory effects that may have an influence on the data. These 
include standard user and item rating effects as well as popularity ---referred
to as ``support'' in \cite{bell2007scalable}--- and time-related
effects. Interestingly, \cite{bell2007scalable} uses successive Bayesian
regressions for each effect to reduce the variability inherent in using a
linear model with many missing observations. However, our finding that the item
rating and the popularity effects are strongly positively associated suggests
that these should not be considered as uncorrelated effects that can be be
successively eliminated in a linear regression model.

A more comprehensive model of an informative selection bias in recommendation
has been investigated earlier by \cite{marlin2003modeling} who proposed to use
generative probabilistic models involving both the observed ratings as well as
latent variables. The latent variables correspond to unobservable explanatory
variables that can influence both the fact that a particular rating is
available as well as its value, allowing to model Not Missing At Random (NMAR)
data in the sense of \cite{rubin1987statistical}. In
\cite{marlin2012collaborative}, the authors report the results of experiments
that support the necessity of such a modeling by showing a significant
mismatch between the empirical distributions of voluntary ratings of
user-selected songs (listened through Yahoo! Music's LaunchCast Radio) and
ratings of system-selected songs by users recruited to participate to the
experiment.

In contrast to \cite{marlin2003modeling}, we do not intend to model
explicitly the missing data mechanism nor to introduce latent variables
representing implicit categories of the population of users. The proposed
approach consists in deriving a simple regularization (or penalty) term that
incorporates some general knowledge about the selection bias. This
regularization term can be used with any recommendation method expressed as
the solution of a variational criterion that involves the true population
average of the ratings. In the context of this paper, we only consider the case
where the regularization term is used to improve the estimation of movie
ratings from small samples of the population in neighborhood-based prediction.

\begin{figure*} \centering
	\subfigure[MovieLens1M]{\includegraphics[width=5cm]{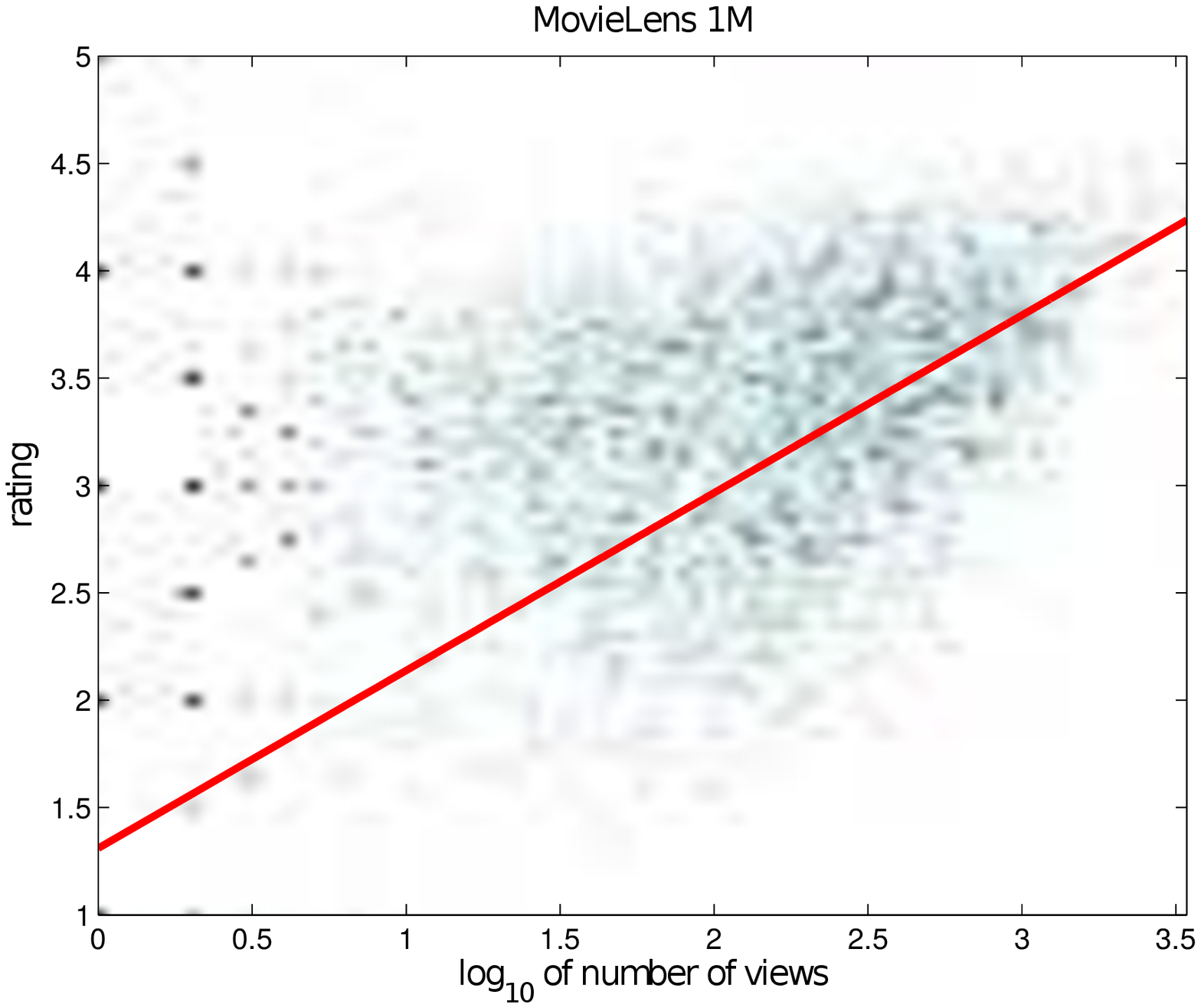}\label{ML1M-obs}}
	\subfigure[MovieLens10M]{\includegraphics[width=5cm]{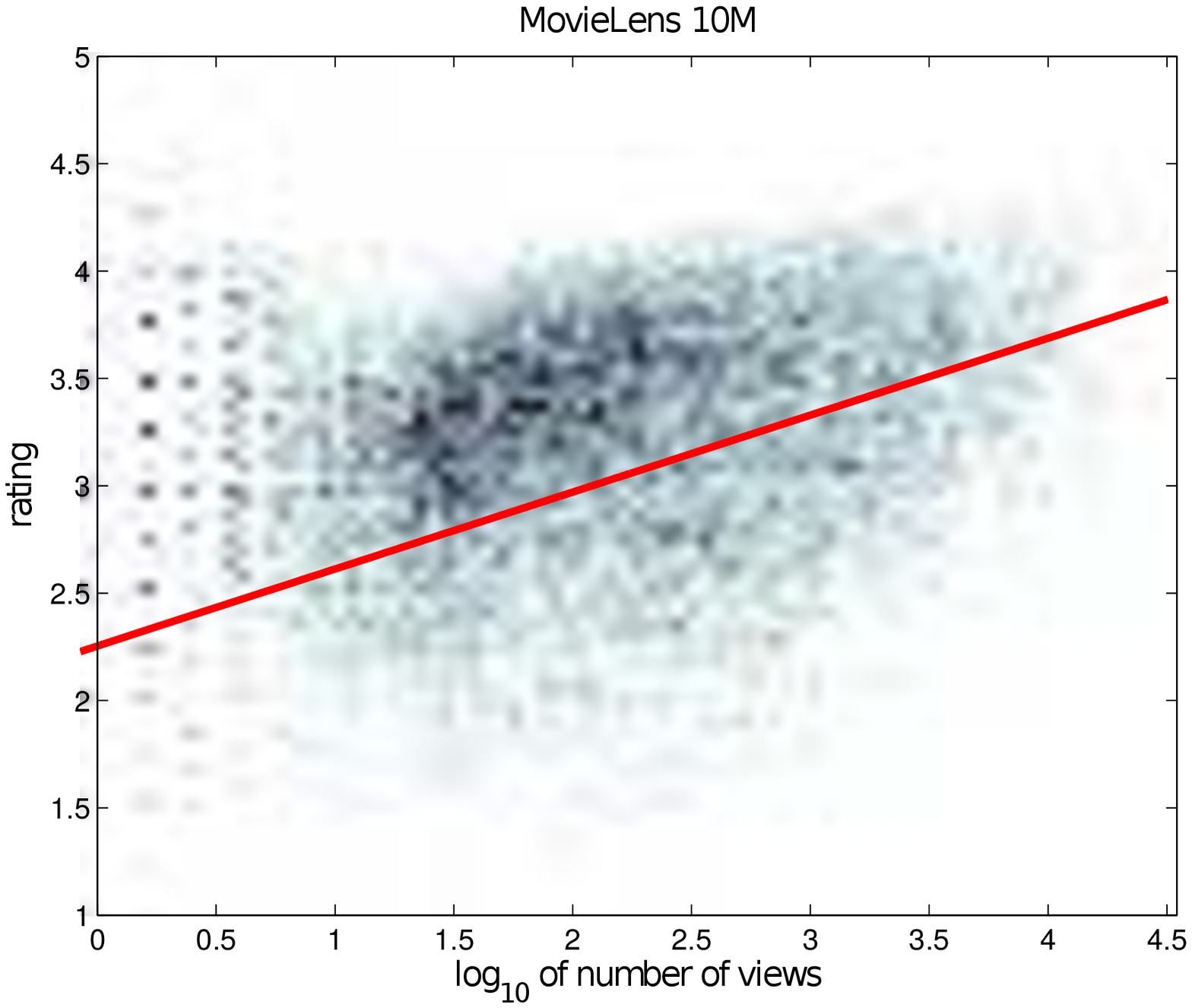}\label{ML10M-obs}}
	\subfigure[Netflix]{\includegraphics[width=5cm]{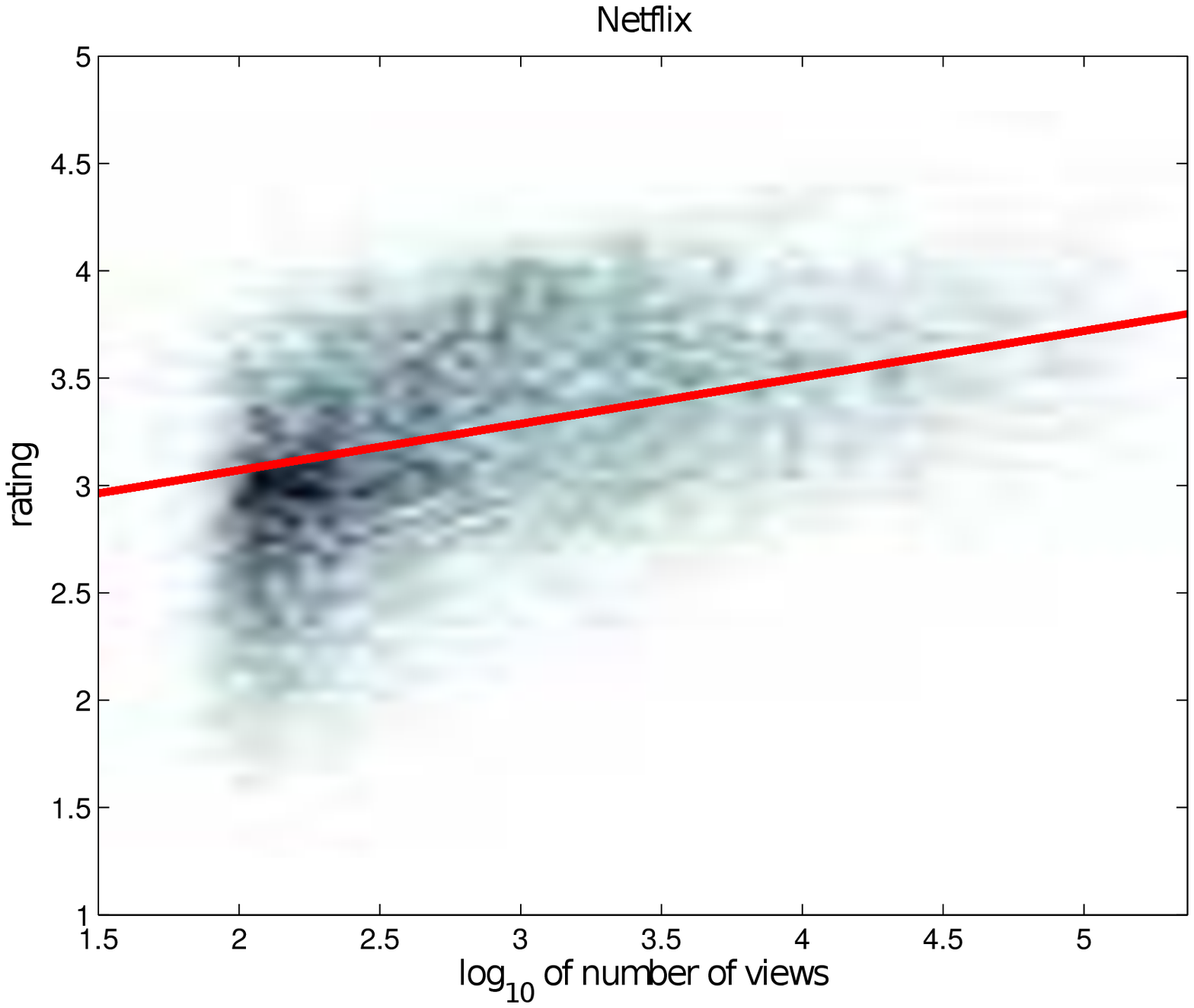}\label{Netflix-obs}}
	\caption{Scatterplot of average movie rating as function of the base ten
		logarithm of the number of views. The red line indicates the fitted
		regression curve. From left to the right: MovieLens 1M, MovieLens 10M and
		Netflix.}
	\label{fig:obs_all_datasets}
\end{figure*}

In Section \ref{sec:obs-data}, we provide a quantification of the selection
bias phenomenon on the MovieLens and Netflix datasets. Section \ref{sec:model}
describes our variational approach for learning ratings taking into account
this selection bias. Sections~\ref{sec:simu-data} and~\ref{expe_real_data}
provide, respectively, numerical experiments on simulated data and on the MovieLens dataset.

\section{Characterizing the Selection Bias \label{sec:obs-data}}

In this Section, we present statistical observations made on well-known movie
rating datasets showing significant positive association between
the popularity of movies and their ratings.

Figure~\ref{fig:obs_all_datasets} displays scatterplots of movie average rating
as a function of the number of ratings corresponding to, from left to right,
the MovieLens 1M (6k users, 3.7k movies, 1M ratings); MovieLens 10M (70k users,
11k movies, 10M ratings); and, Netflix (480k users, 17.8k movies, 100M ratings)
datasets. The y-scales of the three subplots of Fig.~\ref{fig:obs_all_datasets}
are directly comparable and correspond to a 1-5 scale where 5 corresponds to the
highest possible rating\footnote{MovieLens ratings are half integers from .5 to
	5 and Netflix ratings are integer-valued from 1 to 5. Given the variance of
	the ratings ---which is about 1---, the difference between both sorts of
	ratings is not significant. In the following we treat the ratings as
	continuous Gaussian random variables with homoscedastic variance.}. The
number of ratings are plotted on the x-scale using a base ten log scale. 
In particular, the number of ratings obtained by the most rated movies is observed
to be roughly proportional
to the overall number of ratings, which increases by a factor ten when going
from one plot to the next. It is important to keep in mind that the leftmost
part of each plot ---especially for the two MovieLens datasets--- corresponds
to very small samples (movies with less than ten ratings when the x-value is
smaller than one) and should not be considered as individually reliable. From
Fig.~\ref{fig:obs_all_datasets}, it is observed that despite their
differences\footnote{One difference is that all movies included in the Nextflix
	dataset have been rated by at least fifty users, as can be observed on the
	rightmost subplot of Fig.~\ref{fig:obs_all_datasets}.}, the three datasets do
show the same general pattern that the highest rated movies are also among the
most ``popular'' ones (i.e., those with highest numbers of ratings). Given the relatively small spread of ratings on the y-axis
(for MovieLens 10M for instance, half of the movie ratings are between 2.8 and
3.6), this positive association between the popularity of movies and their
rating is rather significant; we will refer to this phenomenon as the
\emph{selection bias}.

The observed positive association does not imply a direct causality
relationship. However, it is to be expected that this observation applies, to
some extent, to all situations where rating an item is a voluntary action taken
by the user.
The important point is that the popularity of an item
is in itself a valuable information even when the objective is to estimate the
population average of its rating. To illustrate this fact, assume that movies A
and B have comparable average ratings but that these ratings were obtained from
$n_A$ and $n_B$ users, where $n_B$ is larger than $n_A$. Standard statistical
arguments suggest that the rating of item $B$ is more reliable than that of
item $A$ because it has been estimated from more users. The positive association
reinforces this observation by making the hypothesis that the actual population
rating of $B$ exceeds that of $A$ more likely, as it has been selected more
often. This effect will be most likely negligible if $n_A$ is itself large as
the statistical error in estimating the population rating of item A is small
anyway in this case. On the other hand, when dealing with \emph{small samples}
---when $n_A$ is, say, less than fifty---, taking into account the selection
bias can be significant. The issue of small samples is inherent to
recommendation. For instance, 32\% of the 10k movies of MovieLens 10M have been
rated by less than 50 users, despite the fact that the most popular movie has
been rated by about half of the users. When one wants to go
beyond recommending the highest rated items based on the whole population, the
explicit or implicit use of sub-populations of users will necessitate a proper
handling or small samples, even when the complete database is very large.

In the rest of this section, our objective is twofold. First, we make
the previous comments more formal by providing a quantitative measure of the
association between the popularity and the rating. Next, our aim is to do so in
a way which can be exploited in a computationally efficient manner to improve
the estimation of movie rating based on small populations of users. This
objective will be addressed by use of linear regression.

Assuming the population can be considered as homogeneous, we define the
observations as the pairs $(X_{t},Y_{t})_{t=1,..,n}$, where $X_{t}\in\{1,\ldots
K\}$ is the selected movie and $Y_{t} \in \mathbb{R}$ denotes the rating of the
movie. $n$ and $K$ refer, respectively, to the total number of ratings and to
the number of rated items. We further assume that pairs can be considered as
independent and identically distributed so that the statistical model is
parameterized by
\begin{equation}
	\begin{cases}
		\theta_{k} = \operatorname{E}(Y_t|X_t = k) & \text{(expected item rating)}\\
		\lambda_k = \operatorname{P}(X_t = k) & \text{(item selection probability)}
	\end{cases}
\end{equation}
for $k=1,\dots,K$. Sufficient statistics for this model are given by
$$N_k=\sum_{t=1}^{n}\mathbf{1}_{\{X_{t}=k\}}$$ and
$$S_k=\sum_{t=1}^{n}Y_{t}\mathbf{1}_{\{X_{t}=k\}}$$
which are, respectively, the number of times item $k$ has been rated as well as the cumulated sum of its ratings.

The subplots of Fig.~\ref{fig:obs_all_datasets} display $S_k/N_k$ as a
function of $\log_{10}(N_k)$, for all items $k=1,\dots,K$. As noted earlier,
both the x- and y- values of this scatterplot correspond to statistics computed
from data and should thus be considered as noisy, which is clearly visible in
the left-hand part of each subplot. To account for this fact we use weighted
Total Least-Squares (TLS) -- or Deming -- regression to fit a symmetrized form of
the linear regression curve to the scatterplot. More precisely, denoting by
$x_k = \ln(N_k/n)$ and $y_k = S_k/N_k$ for $k=1,\dots,K$ we fit
$(a,b)$ by minimizing
\begin{equation}
	\sum_{k=1}^K (x_k - \hat{x}_k)^2/v_k + (y_k - \hat{y}_k)^2/w_k
	\label{eq:wtls}
\end{equation}
where $(\hat{x}_k, \hat{y}_k)$ is the orthogonal projection of $(x_k, y_k)$ 
onto the straight
line $y = a x + b$. Standard asymptotic statistical arguments show that, as $n$ tends to infinity, 
\begin{align*}
	& \sqrt{N_k}(x_k - \ln\lambda_k) \Rightarrow \mathcal{N}(0, 1 -\lambda_k) \\
	& \sqrt{N_k}(y_k - \theta_k) \Rightarrow \mathcal{N}(0, \sigma^2)
\end{align*}
where $\Rightarrow$ corresponds to convergence in distribution,
$\mathcal{N}(\mu,\upsilon)$ denote the Gaussian distribution with mean $\mu$
and variance $\upsilon$, and, $\sigma^2$ is the common variance of the ratings,
which can be estimated from the data. In light of these results and given the fact that most of the movies have a selection probability $\lambda_k$ that is much smaller than 1, we used as weights
\[
v_k = 1/N_k \quad \text{and} \quad w_k = \sigma^2/N_k
\]
The main effect of the weights $v_k$ and $w_k$ in~\eqref{eq:wtls} is to focus
the estimation on the most popular movies, whose average ratings should be
considered as being more precisely estimated. Note that the use of TLS also
makes the problem symmetric and one would obtain the same linear fit by
permuting the data associated to the x- and y- axis (which is of course not
true for standard linear regression which assumes that the data on the x-axis
is observed without noise). The weighted TLS regression estimate of $a$ and $b$
is obtained using the implementation of \cite{krystek2007weighted}.

\begin{table}
	\renewcommand{\arraystretch}{1.3}
	
	\label{tab:coeff_reg}
	\centering
	
	{\footnotesize \begin{tabular}{|p{5em}||c|c|c|}
		\hline
		Data & MovieLens 1M & MovieLens 10M & Netflix\\ \hline
		full dataset & 0.36 & 0.16 & 0.09 \\ \hline
		random subsets of size 100 & - & 0.27 & 0.15 \\ \hline
		
	\end{tabular}}
	\caption{Slope $a$ estimated by weighted TLS on the three datasets MovieLens 1M, MovieLens 10M and Netflix (see Fig.~\ref{fig:obs_all_datasets}) and median slopes estimated on random subsets of MovieLens 10M and Netflix (see Fig.~\ref{fig:subsampling}).}
\end{table}

The first row of Table~\ref{tab:coeff_reg} reports the values of $a$ estimated by
the above method for the three datasets. Corresponding regression lines are
shown in red on the three subplots of Fig.~\ref{fig:obs_all_datasets} (note
that the x-axis is there $\log_{10}(N_k)$ to make the interpretation of the
values easier). In all three cases, one obtains a significantly positive slope
$a$. When fitting the regression model on whole datasets it is also observed
that the estimated slope decays with the size of the dataset. This observation
can be related to the fact that as the size of the dataset increases, the
maximal number of ratings obtained by the most popular movies increases
proportionally while the rating scale stays unchanged.

\begin{figure} \centering
	\includegraphics[width=7cm]{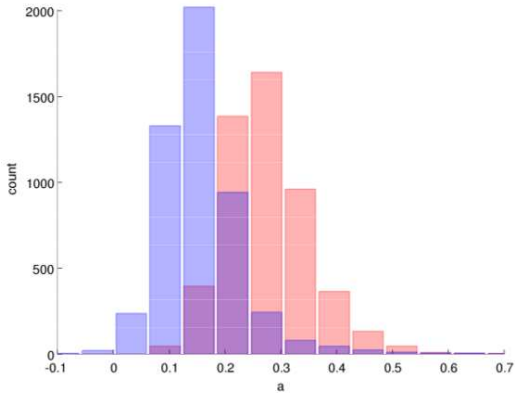}
	\caption{Histogram of slopes $a$ estimated on 5,000 independent random
		subsets of size 100; for MovieLens 10M (red) and Netflix (blue).}
	\label{fig:subsampling}
\end{figure}

This being said, even in very large datasets, recommendation methods will
explicitly (neighborhood-based) or implicitly (matrix factorization methods)
define sub-populations of reduced size that can be considered as
homogeneous. To illustrate the effect of considering sub-populations we
randomly drew $5,000$ random sub-populations from both the MovieLens 10M and
the Netflix datasets. On each of them we fitted a regression model using
weighted TLS. The size of each population was chosen to consist of the ratings
of $100$ users chosen at random. The selections of these users typically
correspond to subsets of 3.5k (MovieLens 10M) or 4.9k (Netflix) movies, with
number of ratings in the range 1--50, where about 43\% of the movies in the selection
have a single rating. The values of $a$ fitted on these sub-populations are
displayed as histograms on Figure~\ref{fig:subsampling} and the median value of
each histogram is reported in the second row of Table~\ref{tab:coeff_reg}. This
experiment shows that values of $a$ appropriate for sub-populations are higher
and usually in the range 0.1--0.35.

To understand the implication of these numbers, consider again our running
example: if $a=0.25$, it means that if we know that movie $A$ has been viewed
by $n_A$ users and movie $B$ by $n_B$ users, where $n_B$ is twice as large as
$n_A$, and in the absence of any other data, we should expect the average
rating of $B$ to be $0.25\ln(2) \approx 0.17$ higher than that of $B$. This
information is non-negligible, corresponding to roughly one sixth of the
standard deviation $\sigma$ of the ratings. In Section~\ref{expe_real_data}, we
will show that knowledge of this selection bias can indeed be leveraged to
improve the prediction of ratings in realistic scenarios.

\section{Using the Selection Bias as a Regularizer \label{sec:model}}

In this section, we derive an estimation criterion that corresponds to a
regularized likelihood estimator, where the regularization accounts for the
selection bias. We show that optimizing this criterion, which is both
smooth and convex, can be done efficiently using standard numerical optimization
tools.

\subsection{Variational Model}

As indicated before, we model the ratings of movie $k$ by a Gaussian
distribution
\[
p(Y\vert X=k;\theta)\sim\mathcal{N}(\theta_{k};\sigma^{2})
\]
where $\sigma^{2}$ is a fixed variance for all movies and $\theta = (\theta_k)_{k=1,\dots,K}$ is the vector of expected movie ratings. This is clearly not the
only option and the method could also be applied using the logit link function,
as in \cite{davenport20141, lafond2014probabilistic}, if we were given binary (``like/dislike'') ratings.

For, the selection probabilities $\lambda=(\lambda_k)_{k=1,\dots,K}$ it is
important to consider the logistic form of the multinomial distribution
parameterized by a vector $\beta$:
\[
\lambda_{k}=\frac{e^{\beta_{k}}}{\sum_{j=1}^{K}e^{\beta_{j}}}
\]
Up to a shift, $\beta_k$ is homogeneous to $\ln\lambda_k$ and hence to the
quantity displayed on the x-axis of Fig.~\ref{fig:obs_all_datasets}. Note that
the vector $\beta$ itself is only identifiable up to a shift as replacing all
$\beta_k$ by $\beta_k + \delta$ would leave the vector $\lambda$ of
probabilities unchanged due to the logistic normalization term. This lack of
identifiability is not a problem as estimating $\beta$ is not required and we
will treat $\beta$ as a so-called ``nuisance parameter'', optimizing it over
all possible values.

Our model of the selection bias thus relies on the assumption that $\beta$
---the log-probability of selecting each item (up to a constant)--- must be
close to $a\theta+b$ where $a$ and $b$ are global parameters that quantify the
selection bias, estimated following the method exposed in Section
\ref{sec:obs-data}. Due to the non-identifiability of $\beta$, we can disregard
the intercept $b$ and let $\beta$ be determined in the shift direction by the
value of $\theta$.

Thus, the problem of estimating $\theta$ boils down to jointly minimizing over $\theta$ and
$\beta$ the following cost function :
\begin{align}
	f\left((X_t,Y_t)_{t=1}^{n};\theta,\beta \right) = \, &  L_{1}\left(\left(Y_{t}\vert
	X_{t}\right)_{t=1}^{n};\theta\right) \nonumber \\ 
	& + L_{2}\left(\left(X_{t}\right)_{t=1}^{n};\beta\right) \nonumber \\ 
	& + r\Vert\theta-a\beta\Vert_{2}^{2}\label{eq3}
\end{align}

where
\begin{enumerate}
	\item $L_{1}$ is the negative conditional log-likelihood of the Gaussian model of the observed ratings;
	\item $L_{2}$  is the negative log-likelihood of the marginal distribution of the selections $(X_t)$;
	\item the last term is the regularization that
	constrains the ratings $\theta$ 
	to stay close to $a\beta$, the log-probability of selection scaled by
	$a$.
\end{enumerate}

The parameter $r > 0$ controls the influence of the regularization term and will typically be set using cross-validation on the training data.

\subsection{Inference Algorithm \label{subsec: algo}}

We first rewrite $L_{1}$ and $L_{2}$ :
\[
L_{1}=\sum_{k=1}^{K}\sum_{t=1}^{n}\mathbf{1}_{\{X_{t}=k\}}\frac{\left(Y_{t}-\theta_{k}\right)^{2}}{2\sigma^{2}}
\]
\[
L_{2}=-\sum_{k=1}^{K}\sum_{t=1}^{n}\mathbf{1}_{\{X_{t}=j\}}\beta_{k}+n\log\left(\sum_{j=1}^{K}e^{\beta_{j}}\right)
\]

Using the notations $N_k$ and $S_k$ introduced in Section~\ref{sec:obs-data}, one obtains
\[
L_{1}=\sum_{k=1}^{K}\frac{\theta_{k}^{2}}{2\sigma^{2}}N_k-\frac{\theta_{k}}{\sigma^{2}}S_k+C
\]
and
\[
L_{2}=n\log\left(\sum_{j=1}^{K}e^{\beta_{j}}\right)-\sum_{k=1}^K N_k\beta_{k}
\]
$C$ being a constant that does not depends on $\theta$.

The gradient of the objective function $f$ is the concatenation of the gradients with respect to $\theta$ and
$\beta$:
\[ \nabla f=\left[\begin{array}{c} \nabla_{\theta}f\\ \nabla_{\beta}f
\end{array}\right]\in\mathbb{R}^{2K}
\]
where
$\nabla_{\theta}f(k)=\frac{\theta_{k}}{\sigma^{2}}N_k-\frac{S_k}{\sigma^2}+2r(\theta_k-a\beta_k)$
and
$\nabla_{\beta}f(k)=-N_k+n\frac{e^{\beta_{k}}}{\sum_{j}e^{\beta_{j}}}-2ra(\theta-a\beta)$.

The Hessian has the following block structure 
\[ Hf=\left[\begin{array}{cc} H_{\theta\theta} & H_{\theta\beta}\\
H_{\beta\theta} & H_{\beta\beta}
\end{array}\right]
\]
where $H_{\theta\theta}=diag\left(N_k/\sigma^{2}+2r,k=1...K\right)$,
$H_{\theta\beta}=H_{\beta\theta}=diag(-2ra)$ and
\[ [H_{\beta\beta}]_{ij}=\begin{cases} n\lambda_{i}(1-\lambda_{i}) +2ra^2 &
\mbox{if }i=j\\ -n\lambda_{i}\lambda_{j} & \mbox{otherwise}
\end{cases}
\]

\begin{prop}
	Assuming that all counts $(N_k)_{k=1,\dots,K}$ are strictly positive, the criterion $$ f\left(
	(X_t,Y_t)_{t=1}^N\vert \theta , \beta \right) $$ is strictly convex with
	respect to $\left(\theta, \beta \right)$.
\end{prop}

\begin{proof}
	Assuming $N_k > 0$, $L_1$ and $L_2$ are known to be strictly convex wrt. to,
	respectively, $\theta$ and $\beta$ up to the already mentioned
	identifiability issue for $\beta$ (the Hessian of $L_2$ has $\beta = (1,
	\dots,1)^T$ as null direction).
	
	For the regularization term, the function $$(\theta, \beta)\in
	\mathbb{R}^K\times \mathbb{R}^K \mapsto \Vert \theta -a\beta \Vert^2 $$ being
	separable in $k$, it is sufficient to consider the case where $K=1$, i.e when
	$\left(\theta,\beta \right)\in \mathbb{R}^2 $. In that case, the Hessian of the regularization term reduces to the following 2 by 2 matrix:
	\[ H(\theta, \beta)=\left[\begin{array}{cc} 2 & -2a  \\
	-2a & 2a^2
	\end{array}\right]
	\]
	The eigenvalues of this matrix are $0$ and $2(a^2+1)$ and the associated
	eigenvectors are respectively $(a,1)$ and $(-a,1)$. Hence the Hessian of
	$\Vert \theta -a\beta \Vert^2$ is a positive matrix. Its $K$ null directions
	(vectors of the form $\theta=(0, \dots, 0, a, 0, \dots, 0), \beta=(0, \dots,
	0, 1, 0, \dots, 0)$) do not span the null direction of $L_2$ and hence $Hf$
	is positive definite.
\end{proof}

The parameter inference can thus be performed using fast converging algorithms
like Newton-Raphson. In most scenarios however, the large size of the Hessian
matrix ---equal to twice the number of items--- makes its storage and
inversion cumbersome. For the experiments reported in the following we thus
used open source implementations of the Limited-memory BFGS (L-BFGS)
approach that yields comparable performance using a very small memory footprint
(typically of the order of ten times the number of items in our case).

\section{Validation: experiments on simulated data \label{sec:simu-data}}

To illustrate the behavior of the method, we start by considering a small-scale
simulated scenario in which the data is generated using a probabilistic model
related to the variational criterion proposed in~(\ref{eq3}). We first show
that in this situation, the proposed approach ---referred to as SB (for
Selection Bias) in the following--- is indeed able to recover the true mean
rating parameter more efficiently than the least squares (LS) estimator that
estimates $\theta_k$ directly by the empirical average $S_k/N_k$\footnote{Note that the
	LS estimator may also be interpreted as the solution of~(\ref{eq3}) when the
	regularization parameter $r$ is set to zero.}. We also
discuss the robustness of the approach with respect to the parameters $a$, $r$,
as well as, the size $n$ of the sample.

For simulating the data, we select an arbitrary vector of mean ratings
$\theta^*$ of size $K$ and generate the associated vectors of logistic
parameters according to $\beta^* \sim \mathcal{N}\left(\theta^*/a,\sigma_b
\right)$. The data is then simulated according to
\begin{align*}
	& X_t\sim \mathcal{M}ult(\lambda^*)\\
	& Y_t\vert X_t=k \sim\mathcal{N}(\theta^*_{k};\sigma^{2})
\end{align*}
for $t=1,\dots,n$, where $\lambda^*_{k}= e^{\beta^*_{k}}/(\sum_{j=1}^{K}e^{\beta^*_{j}})$.

The parameters of the simulation are selected to be somewhat comparable to the
observations made on real data in Section~\ref{sec:obs-data}: values of
$\theta_k$ in the range 1--5, $a=0.35$, $\sigma=1$. The value of $\sigma_b$
is set to 1. For illustration purpose we use a small value of $K$, $K=3$,
varying the number $n$ of ratings between 20 and 2,000. Note that due to the
relatively high values of $a$ and of the ratings spread (see
Fig.~\ref{fig:bar_theta} below) the best rated, and hence most popular, of the
three items is typically forty times more frequent than the lowest rated
one. Hence, the value of $N_k$ for the lowest rated item is usually rather
small, in the range between 1 to 50, depending on the value of $n$\footnote{To
	ensure that the minimizer of~(\ref{eq3}) is uniquely defined, only the
	simulations for which all $N_k$, for $k=1, \dots, 3$, are strictly positive
	are retained.}.

\begin{figure}[hbt]\centering
	\includegraphics[width=8cm]{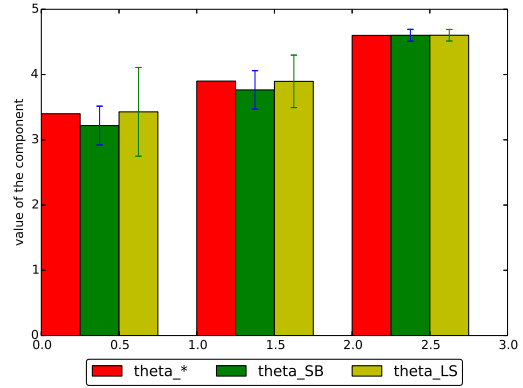}
	\caption{Recovery of the parameter $\theta$: comparison of the Selection Bias
		(SB) and the Least Squares (LS) estimators. Red: true value $\theta^*$
		for $k=1, \dots, 3$; Dark green: mean value estimated by SB; Light green: mean value
		estimated by LS. The vertical whiskers represent the standard deviation of the estimates..}
	\label{fig:bar_theta}
\end{figure}

\subsection{Recovery of ratings}
Figure \ref{fig:bar_theta}, compares the results obtained by the Selection Bias
(SB) and Least Squares (LS) estimators for $n=2,000$ using 200 Monte Carlo
replications of the data $(X_t,Y_t)_{t=1,\dots,n}$. It is observed that the
value corresponding to the item that is simultaneously worst rated and least
selected (corresponding to $k=1$) is slightly under-estimated by the SB
estimator compared to LS, with a standard deviation of the estimator that is
also reduced. For the third item that is both highly rated and very frequent
the difference between both estimators becomes to be negligible. Taking into
account the selection bias in the SB estimator thus produces a downward bias
and reduced variability of the estimated mean rating for infrequent items.

\begin{figure}[hbt] \centering
	\includegraphics[width=8cm]{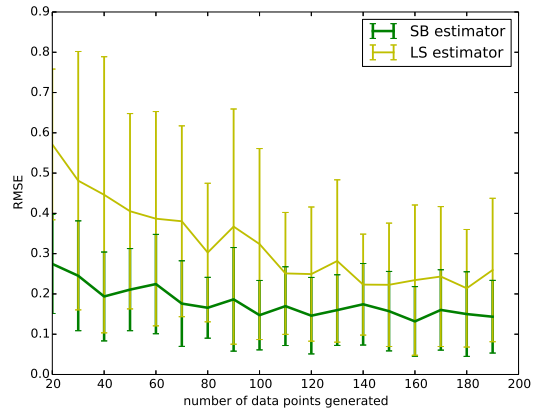}
	\caption{RMSE for SB and LS as a function of $n$, with corresponding error bars.}
	\label{fig:RMSE_N}
\end{figure}

\subsection{Robustness in small samples}

The effect observed on Fig.~\ref{fig:bar_theta} is all the more pronounced that
the sample size $n$ is small. To illustrate this fact, Figure \ref{fig:RMSE_N}
plots the Root Mean Square Error (RMSE) to the true value $\theta^*$, when $n$
increases from $20$ to $200$, computed from 50 Monte Carlo replications of the
data. The RMSE is the square root of the average value of $\Vert
\hat{\theta}-\theta^*\Vert^2$, where $\hat{\theta}$ denotes the estimated value
of $\theta$. Fig~\ref{fig:RMSE_N} confirms that the RMSE of the SB estimator is
always smaller than that of LS and that the gap between them increases for
small sample sizes (when $n$ is small). Note that the variability of the SB
estimator is also much reduced compared to that of LS.

\begin{figure} \centering
	\includegraphics[width=8cm]{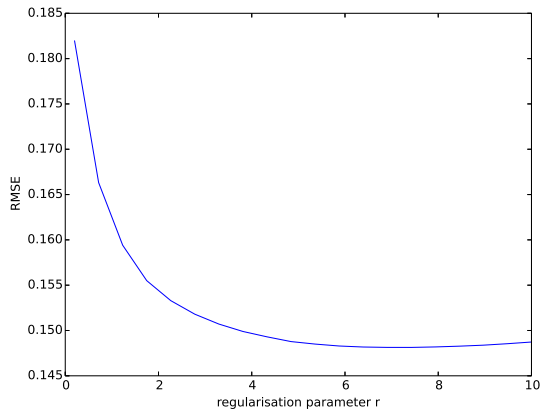}
	\caption{RMSE as a function of the regularization parameter $r$.}
	\label{fig:error_r}
\end{figure}

\subsection{Influence of $r$}
\label{subsec:simu-data:r}

The value selected for $r$ may affect the quality of the results. Figure
\ref{fig:error_r}, shows the RMSE obtained when $n=2,000$ for different values
of $r$ varying between 0.2 and 10. The curve is averaged over 200 Monte Carlo
replications. Here the optimal value of $r$ is $r=7$, which is rather high as
the simulation parameters almost satisfy $\beta^*=a^*\theta^*$, up to the
Gaussian perturbation of variance $\sigma_b^2$. Most importantly, the curve
displayed on Fig.~\ref{fig:error_r} is very smooth around its minimum with a
small curvature, meaning that values of $r$ between 5 and 10 yield, in this
case, a performance comparable to that corresponding to the optimal choice of $r$.

In Section~\ref{expe_real_data}, $r$ will be set by searching for the minimum
value of the RMSE on a validation subset corresponding to a small number of
users.

\begin{figure} \centering
	\includegraphics[width=8cm]{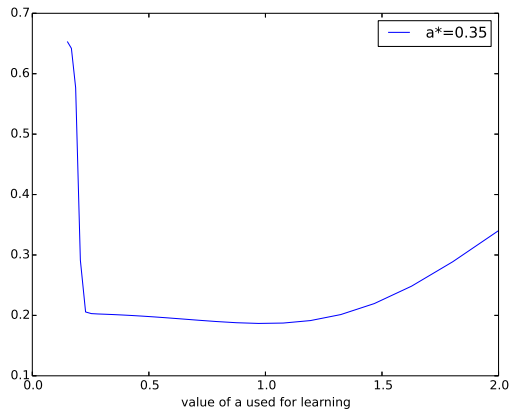}
	\caption{RMSE as a function of the parameter $a$ used in the objective function; the parameter value used to simulate the data is set to $a^*=0.35$.}
	\label{fig:influence_u1}
\end{figure}

\subsection{Influence of $a$}

The previous experiments have been carried out in the idealized setting where
the parameter $a$ that controls the generation of the data is known and used
for the inference. In realistic scenarios, $a$ will be known approximately only
and it is not advisable to try to estimate $a$ together with the other
parameters, in light of the variability observed on
Fig.~\ref{fig:subsampling}. Figure \ref{fig:influence_u1}
shows the RMSE when the parameter $a$ used in the objective
function~(\ref{eq3}) differs from the value $a^*$ used for simulating the data,
which is here fixed to $a^*=0.35$. It is observed that the RMSE barely varies
for $a\in (0.23,1.2)$, indicating that the SB estimator is very robust to the
overestimation of the slope $a$ and only requires that it be set high enough to
perform well.

In the experiments to be described in the next section, we also observed that
the results were very robust to the choice of $a$, with values in the range
0.25--0.5 yielding unnoticeable changes in overall performance.

\section{Experiments On Real Data Using Neighborhood-based Collaborative Filtering \label{expe_real_data}}

In this section, we describe recommendation experiments carried out on the
MovieLens dataset. The inference algorithm described in the previous
section is used to estimate the ratings of films selected by sub-populations of
users. Each sub-population corresponds to the neighborhood of the user for
which one wants to make recommendations. We first discuss the choice of the
similarity measure used to define the neighborhoods.

\subsection{Neighborhood construction}

The penalization scheme introduced in~(\ref{eq3}) is a generic tool designed to
improve the estimation of ratings from small samples by taking into account the
selection bias. In this section, we describe the simplest way in which this
method can be used in the context of recommendation. We consider a standard
neighborhood-based collaborative filtering approach in which~(\ref{eq3}) will
be used only to estimate the movie ratings from the sub-population of users
that belong to each neighborhood. The baseline approach usually considered in
the literature consists in using the empirical averages of the ratings in the
sub-population, that is, LS (Least Squares) following the terminology of
Section~\ref{sec:simu-data}. To allow for a meaningful comparison between the
proposed estimator (termed SB) and LS, we will use the exact same algorithm to
define the relevant sub-population of users that belong to the neighborhood of
the user for which we want to discover new relevant movies.

For the baseline to be significant we will use a state-of-the-art approach for
defining neighborhoods based on data-driven features of reduced
dimension. Dimensionality reduction is a popular method in recommendation.
Among the numerous algorithms that have been proposed to perform matrix
factorization, the Singular Value Decomposition (SVD) was shown quite early to
provide good performance (see for example \cite{bell2007lessons}). For the
experiments, we used an incremental implementation of SVD for the Julia
language that can handle large sparse data matrices \footnote{This library can
	be found at https://github.com/aaw/IncrementalSVD.jl }. Considering the
MovieLens 10M dataset, we first split randomly the data into a training set
containing 75~\% of the ratings of each user and a test set with the remaining
25~\%. A rank-25 SVD of the rating matrix (considering unobserved values as
zeroes) corresponding to training data was computed so as to determine a
representation of each user as a vector of features in a 25-dimensional space.

To define the neighborhood of a user whose feature vector is denoted by $u$,
we used the cosine similarity defined, for another vector $v$, by $c =
\left\langle u,v \right\rangle/\Vert u\Vert \Vert v \Vert$. The size of the
neighborhoods was set to 100 and hence the first 100 vectors $v$ with highest
cosine similarity with $u$ are included in the neighborhood of $u$. The typical
size of these neighborhoods is comparable to that of the random subsets
considered in Section~\ref{sec:obs-data}, that is, $K$ (number of movies viewed
in the neighborhood) and $n$ (number of ratings) of the order of a few
thousands.

Note that it would be very easy to use the similarity $c_t$ corresponding to each item as a weight in~(\ref{eq3}): simply redefining $N_k$ and $S_k$ as
\begin{align*}
	& N_k=\sum_{t=1}^{n}c_t \mathbf{1}_{\{X_{t}=k\}}\\
	& S_k=\sum_{t=1}^{n} c_t Y_{t}\times\mathbf{1}_{\{X_{t}=k\}}
\end{align*}
is equivalent to weighting each observation $(X_t,Y_t)$ in the likelihoods
$L_1$ and $L_2$ by $c_t$ rather than 1. For $r=0$, the solution of~(\ref{eq3})
when using this weighting will be the similarity-weighed average of each item's
ratings rather than the simple average. In our case however, the results
obtained with these similarity-weighted versions of the SB and LS estimators
were not significantly different from those obtained with the basic
(unweighted) versions and are not reported here.

\subsection{Choice of the parameters\label{sub:choice_r}}

The slope parameter $a$ was set to the value estimated in Section \ref{sec:obs-data} on random subsets of 100 users and was kept fixed through the experiments.

\begin{figure} \centering
	\includegraphics[width=8cm]{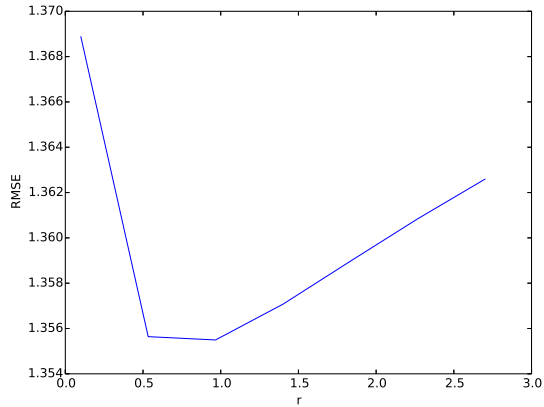}	
	\protect\caption{RMSE as a function of $r$, computed on
		the first 100 users of the base.}
	\label{fig:choosing_r}
	
\end{figure}

For the regularization parameter $r$, some tuning is necessary as observed in
Section~\ref{subsec:simu-data:r}. Here again, the value of $r$ is fixed
globally, using a common value for all neighborhoods, as tuning $r$ on small
populations is definitely not advisable and would require validation data for
each user. For that purpose, we conducted a preliminary experiment on a subset
of 100 users of the base and computed the averaged RMSE (see below) for
different values of $r$ between $0.1$ and $2.5$. The results are displayed on
Figure \ref{fig:choosing_r}. Similarly to the graph of Fig.~\ref{fig:error_r}
computed on simulated data, we observe that the quality of the recommendations
improves whenever $r>0$, with an optimum about $r=1$ which was used in the
following.

\subsection{Evaluation Metrics}

To evaluate the results, we used various metrics focusing on different aspects of the estimation.

The first criterion, RMSE, is aimed at quantifying the calibration of the
rating estimates provided by the algorithms. The RMSE, is the classical metric
that was used in the Netflix Challenge. If $\mathcal{T}$ denotes the set of
indices featured in the test set for a given user, we define
\[
\text{user RMSE} = \sqrt{\frac{1}{\vert \mathcal{T} \vert}\sum_{t\in\mathcal{T}}\left(Y_t - \hat{\theta}_{X_t}  \right)^2}
\] 
where $\hat{\theta}$ denotes the rating estimates determined on the training
data for this user\footnote{For a movie $k$ that was not selected in the user
	neighborhood but that is present in the user test set, we set by convention
	$\hat{\theta}_l = 3.5$, which corresponds to the empirical average of all
	ratings}. The RMSE is the average of the user-level RMSE for all users. It
should be noted that, in contrast to the criterion used in
Section~\ref{sec:simu-data} that weighted all items equally, the RMSE as
defined above does give more weights to the error corresponding to movies that
appear frequently.

In addition, the RMSE gives as much importance to accuracy in predicting
low ratings that it does for high ratings, whereas the latter is arguably more
relevant in the perspective of recommendation. For this reason, we also
consider the precision associated with the recommendation of a few number of
items that are believed to be relevant. Relevant items are defined as the
movies actually selected by the user in the test set \emph{and that were rated
	4 or above} (that is, 4, 4.5 or 5 for the MovieLens datasets).

The first measure is the standard Precision-at-N (denoted P@N) that assumes
that only $N$ items are to be recommended and counts the number of relevant
items (also called true positives) in those $N$ recommended items:
$$\text{P@N} =
\vert \text{relevant items} \vert / N$$
In this case, it is natural to
consider that the $N$ movies recommended by a method are those that have the
highest estimated ratings. However, P@N is fundamentally a ranking measure and
we will see below that due to the selection bias it is optimized by a very
simple heuristic that does not even rely on estimating the movie ratings.

To mitigate this observation, we also consider $\text{P@}\tau$ in which the set
of relevant items is determined by considering the movies for which the
estimated rating is above the threshold $\tau$, that is,
$$\text{P@}\tau =
\vert \text{relevant items} \vert / \vert \text{items with est. rating $\geq \tau$}\vert$$

Although, this second way of proceeding does not explicitly control the size of the recommendation set that corresponds to a given value of $\tau$, it is appropriate to measure the accuracy in predicting, in a calibrated way, high values of the ratings.

\subsection{Results} \label{subsec:results}

We present in this section the results obtained on the MovieLens 10M dataset. We computed RMSE, P@N and $\text{P@}\tau$ by selecting randomly and averaging results over 1,000 users.

\begin{table}[!t]
	\renewcommand{\arraystretch}{1.3}
	\caption{Results of the experiments on the MovieLens 10M dataset.}
	\label{tab:res_ML10M}
	\centering
	
	{\footnotesize \begin{tabular}{|c||c|c|c|c|}
		\hline
		& RMSE & P@N3 &  P@N14 & P@$\tau$4  \\
		\hline
		$\hat{\theta}(SB)$ & $\mathbf{0.923}$ & $\mathbf{0.183}$ & $\mathbf{0.125}$ & $\mathbf{0.0332}$  \\
		\hline
		
		$\hat{\theta}(LS)$ & 0.952  & 0.0022 &   0.0028 & 0.0211 \\
		\hline
		Popularity Ranking & - & 0.239 &  0.160 & -	 \\	
		\hline
	\end{tabular}}
\end{table}

The first column of Table \ref{tab:res_ML10M} shows that the SB estimator
improves the RMSE compared to LS. This improvement is significant, confirming
the good behavior of LS for rating estimation. The order of magnitude of the
improvement is limited but this is mainly due to the weighting by the popularity
of movies inherent to the RMSE computation. For less frequent movies, the
improvement brought by SB is indeed major as will be shown below. The remaining
columns of Table \ref{tab:res_ML10M} report the performance in term of the P@N
and $\text{P@}\tau$ metrics defined in the previous section. The values
selected for $N$ correspond to two realistic use cases that can be of interest
in movie recommendation: suggesting a top-3 short list ($N$=3) or building a
recommendation page on a website containing a human-sized list ($N$=14). For
$\text{P@}\tau$, $\tau=4$ was selected in light of the actual threshold used to
determine relevant items in the test set\footnote{Note however that due to the
	fact that the actual observed ratings are half integers, the value of the
	relevance threshold is not precisely defined.}. These results show that, when
it comes to identifying highly rated items, the SB estimator significantly
outperforms the standard empirical average (or LS) estimator.

Figure~\ref{fig:P@N_ML10M} gives more details by displaying the results
obtained for the P@N metric, for values of $N$ between 3 and 30. It is
important to keep in mind that the P@N metric being a ranking criterion it does
not measure the accuracy in evaluating the ratings but rather the ability to
produce a correct ordering of the movies. The general shape of the performance
curve for the SB estimator (blue curve) in Fig.~\ref{fig:P@N_ML10M} suggests
that it succeeds in putting at the top of the list the most relevant items,
with a precision that decreases as $N$ increases. In contrast, the red curve
that corresponds to the performance of the LS estimator shows that it largely
fails when it comes to ranking items. The fact that as $N$ grows, the P@N
metric increases (slightly) with $N$ for the LS estimator suggests that the
failure of LS comes from the fact that it can attribute high ratings to
irrelevant movies. As discussed in Section~\ref{sec:obs-data}, a significant
fraction of movies (almost half of them) rated in each neighborhood has been
rated only once. The LS estimation for these ``hapax'' items is extremely
noisy, being based on a single occurrence: it suffices that one of these be
rated at 5 (the maximal note) to perturb the highest rank of the recommendation
list. This interpretation is confirmed by the green curve of
Fig.~\ref{fig:P@N_ML10M} that corresponds to the performance of the LS
estimator, when restricted to the movies that were at least select twice in the
neighborhood. By raising the threshold value (above 2) the impact of unreliable
ratings could be further reduced and the ranking performance of the LS
estimator improved, at the price of a reduced diversity of the
recommendations. An alternative would be to use a Bayesian mean estimate, as in
\cite{bell2007scalable}, to shrink the estimates towards the global mean rating
for scarcely observed items. It is remarkable that the SB estimator does not
necessitate any adjustment of this sort: as discussed about
Fig.~\ref{fig:bar_theta}, the presence of the regularization term creates a
downward bias for ratings based on few observations, making it highly unlikely
that these unreliable ratings appear at the top of the ranking
list.

\begin{figure}[hbt] \centering
	\includegraphics[width=8cm]{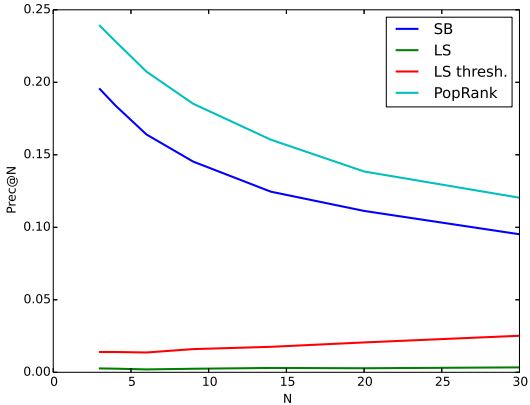}
	\protect\caption{Precision-at-N for $N$ varying in 3-30. The results are averaged over 1,000 users randomly selected among the ML10M database.}
	\label{fig:P@N_ML10M}
\end{figure}

Interestingly, in terms of the P@N metric, the SB estimator is dominated by the
simple heuristic (light blue curve in Fig.~\ref{fig:P@N_ML10M}) that ranks the
movies in the neighborhood according to their popularity and recommend the $N$
most popular ones. The fact that it is possible to recommend the highly rated
movies, without even using the rating data (except for the definition of the
features used to build the neighborhoods) is a clear illustration of the
selection bias phenomenon. This ``ranking by popularity'' approach is also
naturally immune against the variability due to movies with a small number of
ratings. This being said, this strategy is not calibrated and recommends items
whose value is not clearly defined. It is also likely that, in terms of the
diversity of the recommendations, always recommending the most popular
items in each neighborhood is not the optimal approach.

\begin{figure}[hbt]\centering
	\includegraphics[width=8cm]{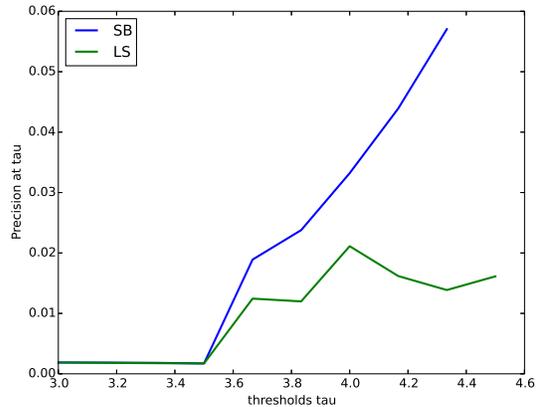}
	\protect\caption{Precision-at-$\tau$ for $\tau$ varying in 3-4.5. The results are averaged over 1,000 users randomly selected among the ML10M database.}
	\label{fig:P@tau_ML10M}
\end{figure}

The alternative consists in measuring the precision at a given threshold $\tau$,
as shown in Figure \ref{fig:P@tau_ML10M}: for thresholds $\tau$ between 3 and
4.5, both the LS and the SB estimators were used to create lists of recommended
items whose estimated rating exceeded $\tau$. The drawback of this approach is
that the size of the recommendation list is not explicitly controlled. In
particular, when interpreting the curves on Fig.~\ref{fig:P@tau_ML10M} it is
important to keep in mind that the size of the recommendation lists
corresponding to the same value of $\tau$ may in fact be different for the two
methods (SB and LS) under consideration. 

It is observed on Figure~\ref{fig:P@tau_ML10M} that up to $\tau=3.5$, which
correspond to the overall average rating, both estimators have comparably low performance as they recommended about half of the items present in the
training set for each neighborhood. When the threshold grows up to $4$, the
lists of recommendations for both SB and LS become more relevant as shown by
the increase of the $\text{P@}\tau$ metric. For the LS estimator however, the
precision values are decreasing for thresholds $\tau$ above 4, showing that a
significant fraction of the highly rated estimates is in fact strongly
contaminated by unreliable values. In contrast, the precision of the SB
estimator keeps improving as $\tau$ increases, showing once again that the
highly rated estimates are much more reliable with the SB approach.

\section{Conclusion}

In this paper, we introduced a model for the link between the probability of
selecting an item and its underlying rating. The corresponding
optimization-based estimator of the underlying rating effectively uses the two
available pieces of information about each item: the empirical frequencies of
selection and the empirical rating averages. The experiments performed on
simulated data showed that in the presence of a selection bias the proposed
estimator provides more reliable estimate of the underlying rating of
highly-rated items. Finally, the approach was used for collaborative filtering
on the MovieLens data showing a large improvement in terms of relevance of the
recommendation.

We believe that the proposed approach can be incorporated in other types of
recommendation algorithms. A first idea would consist in using the selection
bias regularization term, or a variant of it, in matrix completion methods
based on a variational criterion \cite{candes2010matrix}, including the cases
where a different noise model is considered
\cite{davenport20141,lafond2014probabilistic}. Another idea would be to incorporate the selection bias in the prior specification for methods based
on Bayesian modeling \cite{tienmai2015bayesian,gopalan2014bayesian}.

\bibliographystyle{IEEEtran}
\bibliography{biblio.bib}
%

\end{document}